\documentclass[letterpaper, 10 pt, journal, twoside]{IEEEtran} 
\usepackage[utf8]{inputenc}

\usepackage{units}
\usepackage{algorithm, algorithmic}
\usepackage{amsmath}
\usepackage{amssymb}
\usepackage{tabularx}
\usepackage{booktabs}
\usepackage{cite} 

\usepackage{amsthm} 
\newtheorem{lemm}{Lemma}

\usepackage[colorinlistoftodos]{todonotes}

\renewcommand\vec[1]{\boldsymbol{#1}} 
\newcommand{\mat}[1]{\boldsymbol{#1}} 

\newcommand{\de}{\;\mathrm{d}} 

\newcommand{\transp}{^{\top}} 
\DeclareMathOperator{\E}{\mathbb{E}} 

\newif\ifcorrectingmode
\correctingmodefalse

\usepackage[normalem]{ulem}
\ifcorrectingmode
	\newcommand{\revision}[1]{\textcolor{blue}{#1}}
	\newcommand{\deleted}[1]{\textcolor{red}{\ifmmode\text{\sout{\ensuremath{#1}}}\else\sout{#1}\fi}}
	\newcommand{\deletedequation}[2]{\textcolor{red}{\centerline{Removed equation (#1)}}}
\else
	\newcommand{\revision}[1]{#1}
	\newcommand{\deleted}[1]{}
	\newcommand{\deletedequation}[2]{}
\fi

\usepackage[printonlyused]{acronym}
\newacro{DOF}{degree of freedom}
\newacroplural{DOF}{degrees of freedom}
\acrodefplural{DOF}[DOFs]{degrees of freedom}
\newacro{iLQR}{Iterative Linear-Quadratic Regulator}
\newacro{CT}{Control Toolbox}
\newacro{EOM}{equations of motion}
\newacro{OC}{Optimal Control}
\newacro{LQR}{linear-quadratic regulator}
\newacro{PD}{proportional derivative}
\newacro{MPC}{Model Predictive Control}
\newacro{LQ}{linear quadratic}
\newacro{LQOC}{Linear-Quadratic Optimal Control}
\newacro{TO}{Trajectory Optimization}
\newacro{DDP}{Differential Dynamic Programming}
\newacro{COM}{center of mass}
\newacro{COP}{center of pressure}
\newacro{NLP}{nonlinear program}
\newacro{MLP}{Multilayer Perceptron}
\newacro{SLQ}{Sequential Linear-Quadratic}
\newacro{HAA}{hip abduction adduction}
\newacro{AD}{automatic differentiation}
\newacro{HJB}{Hamilton–Jacobi–Bellman}
\newacro{BC}{Behavioral Cloning}
\newacro{IL}{Imitation Learning}
\newacro{RL}{Reinforcement Learning}

\def\TheTitle{MPC-Net: A First Principles Guided Policy Search}

\begin{document}

\title{\TheTitle}
\author{Jan Carius$^{1}$, Farbod Farshidian, and Marco Hutter%
    \thanks{Manuscript received: September 09, 2019; Revised December 14, 2019; Accepted January 22, 2020.}
    \thanks{This paper was recommended for publication by Editor Dongheui Lee upon evaluation of the Associate Editor and Reviewers' comments.
    This work was supported by Intel Labs, the Swiss National Science Foundation (SNF) through project 166232, 188596, the National Centre of Competence in Research Robotics (NCCR Robotics), and the European Union's Horizon 2020 research and innovation program under grant agreement No 780883. This work was conducted as part of ANYmal Research, a community to advance legged robotics.}%
    \thanks{All authors are with the Robotic Systems Lab, ETH Z\"u{}rich, Switzerland. {$^1$\tt\footnotesize jcarius@ethz.ch}}%
    \thanks{Digital Object Identifier (DOI): see top of this page.}
    }

\markboth{IEEE Robotics and Automation Letters. Preprint Version. Accepted January, 2020}
{Carius \MakeLowercase{\textit{et al.}}: \TheTitle}

\maketitle
%
%
%
\begin{abstract}
We present an Imitation Learning approach for the control of dynamical systems with a known model.
Our policy search method is guided by solutions from MPC.
Typical policy search methods of this kind minimize a distance metric between the guiding demonstrations and the learned policy.
Our loss function, however, corresponds to the minimization of the control Hamiltonian, which derives from the principle of optimality.
Therefore, our algorithm directly attempts to solve the optimality conditions with a parameterized class of control laws.
Additionally, the proposed loss function explicitly encodes the constraints of the optimal control problem and we provide numerical evidence that its minimization achieves improved constraint satisfaction.
We train a mixture-of-expert neural network architecture for controlling a quadrupedal robot and show that this policy structure is well suited for such multimodal systems.
The learned policy can successfully stabilize different gaits on the real walking robot from less than 10 min of demonstration data.
\end{abstract}
%
%
\IEEEpeerreviewmaketitle
%
%
\section{Introduction}
The control of robotic systems with fast and unstable dynamics requires carefully designed feedback controllers.
Hybrid, underactuated walking robots pose an especially challenging setting in this respect.

Recent successes in \ac{RL} demonstrate sophisticated walking robot control~\cite{Tan18, Iscen18, Haarnoja18, Hwangbo19, Xie19}, yet a large number of policy rollouts need to be collected to reach the required performance level.
It is, therefore, common practice to use physics simulators during training and subsequently attempt a sim-to-real transfer~\cite{Tan18, Hwangbo19}.

\ac{IL}~\cite{Osa18} appears to be a promising method that could reduce the sampling complexity of learning-based approaches by guiding them with expert demonstrations.
When good demonstrations are available, sampling efficiency can be drastically improved over classical \ac{RL}~\cite{Sun17}.

An appealing way to automatically generate such demonstrations for known dynamical systems are model-based methods such as \ac{OC} and \ac{MPC}.
They provide a formal framework for generating control commands that respect physical constraints and optimize a performance criterion.
Knowledge of a system model and its gradients enable such methods to discover complex robot behaviors in a very sample-efficient way~\cite{Park15, Naveau17, Farshidian17MPC, Winkler18, Neunert18, Carius19}.
Unfortunately, when deploying on a robot, the entire optimization problem has to be solved online because the resulting control policy is only valid around the current state.
Moreover, the robustness against disturbances -- both of intrinsic nature (e.g., modeling errors) as well as external effects -- is critically dependent on the assumption that a new motion plan can be generated sufficiently fast.
Even for moderately complex systems, the update frequency of \ac{MPC} becomes a limiting factor when deploying on onboard computers.

Learning from \ac{OC} solutions has proven a viable option for robot control that combines the advantages of both approaches~\cite{Ratliff06, Abbeel10, Mordatch14, Levine13, Levine14, Kahn17, Choudhury17, Yang19}.
The benefit of using a solver as expert demonstrator over humans or animals is that there is no domain adaptation problem, and one can query demonstrations from arbitrary states.
Additionally, one may request the solver to explicitly handle constraints instead of only presuming that demonstrations are constraint-consistent.

Several methods take an inverse \ac{OC} approach to \ac{IL}:
Multiple local approximations of the value function, computed by \ac{OC} runs, are aggregated into a single global approximation~\cite{Atkeson02, Zhong13, Mansard18}.
The learned value function and its induced optimal policy are in turn used to reduce the \ac{OC} time horizon or speed up convergence.
Alternatively, a \ac{BC} approach to \ac{IL} attempts to directly learn a policy that reproduces the expert's demonstrations without maintaining a value function explicitly.
Accordingly, the original \ac{RL} problem is transformed into a supervised learning problem since the demonstrator's actions can be interpreted as labels.

Our proposed algorithm belongs to the family of such actor-only approaches:
We introduce MPC-Net, a \emph{policy search} method that is \emph{guided} by an \ac{MPC} algorithm to find a parametrized control policy.
The method can be seen as a policy iteration scheme that draws data from a perfect critic (i.e., the \ac{MPC}).
Our key innovation is a theoretically motivated loss function, which is based on \emph{first principles} from \ac{OC}, namely the minimization of the control Hamiltonian.
The structure of the control Hamiltonian captures the system dynamics and constraints of the control problem.
We show that this learning objective has favorable properties in terms of convergence and constraint satisfaction, which is particularly important for systems interacting with the environment.

Closely related to our algorithm are policy search methods with a teacher-learner setup~\cite{Levine13, Levine14, Kahn17}.
These works employ an \ac{OC} solver as a teacher from which a policy is learned.
Contrary to our work, however, the teacher adapts to the student.
This assimilation is achieved by adding a penalty term to the \ac{OC} cost function so that demonstrations are created that remain close to the student's policy.
Additionally, the student's objective is usually the optimization of a distance metric between student's and teacher's policy outputs.
However, minimizing a distance measure may not directly correspond to improved performance, e.g., in constrained settings it is usually more important to satisfy constraints rather than mimicking the teacher accurately.
In our approach, no such choice of a distance metric has to be made.
Notably, our learner is never presented with the optimal control input.
Additionally, since our demonstrator does not adapt to the current policy of the learner, all demonstration samples remain valid and can be re-used, thereby boosting sampling efficiency.

Imitating a demonstrator that is not adaptive to the learner induces the problem of distribution matching:
Inevitable approximation errors between the learned and demonstrated policies make rollouts of the learned policy encounter a different distribution of states than the one from demonstration data.
Ross et al.~\cite{Ross10, Ross11} show that the resulting errors can compound quadratically in the time horizon.
We use elements of their proposed solutions (i.e., probabilistic mixing and dataset augmentation) to ensure that the distributions match.
Simply put, we bias the demonstrator's query states towards the observations that our policy sees and thereby receive samples that match the learner's distribution better.

While the idea of policy search through minimization of the control Hamiltonian applies to arbitrary parameterized policies such as neural networks, weighted motion primitives, or spline coefficients, we consider the very general class of mixture-of-expert neural networks policies~\cite{Jacobs91} in this work.
Our choice caters for the fact that \ac{OC} is an inverse problem with potentially multiple solutions for the same observation.
The expert data may, therefore, exhibit such multimodal behavior.
We show that this choice of network structure has favorable properties in terms of convergence and constraint satisfaction and is particularly suitable for controlling legged robots since these systems inherently exhibit multi-modal dynamics.

\subsection*{Statement of Contributions}
The contributions of this work are as follows:
\begin{itemize}
    \item Derivation of a novel loss function for policy search that is based on fundamental concepts from \ac{OC}
    \item Experiments showing that the explicit enconding of constraints in our loss function achieves improved constraint satisfaction compared to standard behavioral cloning
    \item Demonstration of improved efficiency in terms of \ac{MPC} calls by exploiting a local approximation of the value function
    \item Results showing that a mixture-of-expert network architecture outperforms a general \ac{MLP} for control of a walking robot
    \item Validation of the trained control policies on robotic hardware. The learned controllers successfully stabilize two different gaits on a quadrupedal robot
\end{itemize}

%
\section{Method}
The key steps of our method are listed in Alg.~\ref{alg:policy_learning}\footnote{Our implementation is openly available at \texttt{https://github.com/leggedrobotics/MPC-Net}} and are schematically shown in Fig.~\ref{fig:method_schematic}.
Data is generated by running \ac{MPC} from a feasible, random initial state.
Samples from the resulting optimal trajectories are stored in a replay buffer.
At each policy update step, we construct a loss function by drawing a batch of the stored samples and perform a stochastic gradient descent step in the policy parameter space.
Every mpcDecimation-th iteration, \ac{MPC} produces a new set of samples that augment the dataset.\newline
In this section, we first explain the control problem and the structure of its solution.
Subsequently, we present the theoretical properties of the optimal solution and how they motivate our loss function.
Finally, we show how a neural network policy is trained from demonstrations of optimal trajectories.
\begin{figure}
    \centering
    \includegraphics[width=0.95\columnwidth]{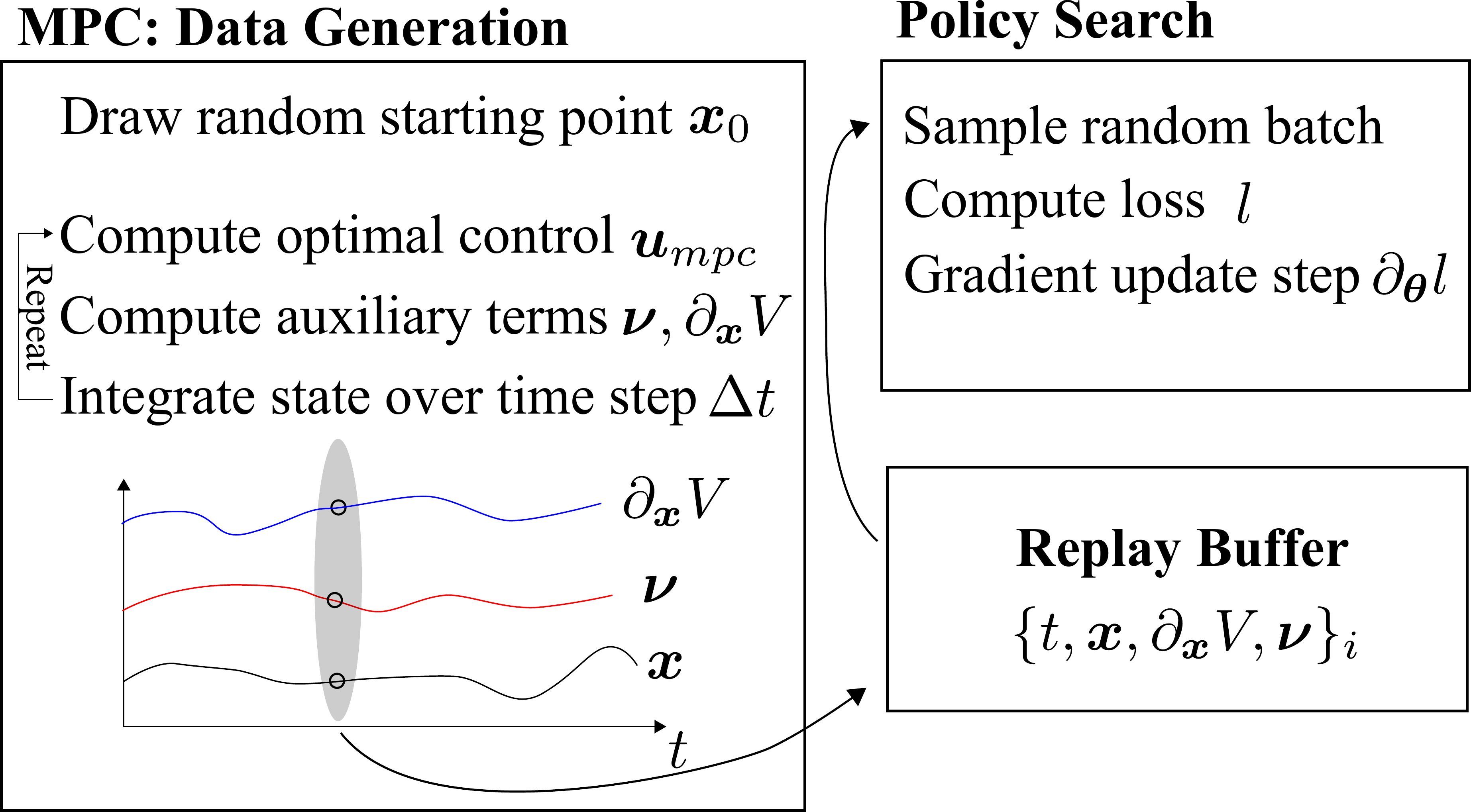}
    \caption{Schematic of the MPC-Net policy learning approach}
    \label{fig:method_schematic}
\end{figure}
\begin{algorithm}[tpb]
    \caption{MPC-Net Guided Policy Learning}
    \label{alg:policy_learning}
    \begin{algorithmic}[1]
        \addtolength{\algorithmicindent}{-1pt}
        \STATE \textbf{Given:} Replay Buffer $\mat{M}$, mpcSolver
        \STATE \textbf{Given hyperparameters:} maxIter, mpcDecimation, batchSize, learningRate, rolloutLength
        \FOR{iter in [1 : maxIter]}
            \IF{modulo(iter, mpcDecimation)}
                \STATE $\alpha \gets 1 - $ iter / maxIter
                \STATE $\vec{x}_0 \gets$ sampleRandomStartingState
                \FOR{$t_i$ in [0, rolloutLength]}
                \STATE $\vec{u}_{mpc}, \mat{K} \gets $ mpcSolver($t_i$, $\vec{x}_0$)
                    \STATE $\vec{x} \gets $ sampleInNeighborhood($\vec{x}_0$)
                    \STATE $\partial_{\vec{x}}V \gets $ valueFunctionDerivative($t_i, \vec{x}$)
                    \STATE $\vec{\nu} \gets $ constraintLagrangian($t_i, \vec{x}$)
                    \STATE Append sample $\{t_i, \vec{x}, \partial_{\vec{x}}V, \vec{\nu}\}$ to $\mat{M}$
                    \STATE $\vec{x}_0 \gets$ stepSystem($\vec{x}_0, \alpha \vec{\pi}_{\text{mpc}} + (1-\alpha) \vec{\pi}(\vec{\theta}_{\text{iter}}, \vec{x}_0))$
                \ENDFOR
            \ENDIF
            \STATE $\mat{S}$ $\gets$ drawRandomSampleBatch($\mat{M}$, batchSize)
            \STATE $\mat{U} \gets $ evaluatePolicyOnSamples($\vec{\pi}(\vec{\theta}_\text{iter}), \mat{S}$)
            \STATE $l \gets$ computeLoss($\mat{U}, \mat{S}$)
            \STATE $\vec{\theta}_{\text{iter} + 1} \gets $ stepOptimizer($\partial_{\vec{\theta}}l$)
        \ENDFOR
    \end{algorithmic}
\end{algorithm}
\subsection{Model Predictive Control}\label{sec:MPC}
We consider a continuous-time, finite horizon \ac{OC} problem
\begin{align}
& \underset{\vec{u} (\cdot)}{\text{minimize}}
& & \Phi(\vec{x} (t_{f})) + \int_{0}^{t_{f}} l(\vec{x}, \vec{u}, t) \de t ,
\label{eq:cost}
\\
& \text{subject to}
& & \dot{\vec{x}} =  \vec{f}(\vec{x}, \vec{u}, t), \quad \vec{x}(0) = \vec{x}_0,
\notag
\\
& & & \vec{g}(\vec{x}, \vec{u}, t) = \vec{0},
\notag
\\
& & & \vec{h}(\vec{x}, \vec{u}, t) \geq \vec{0},
\label{eq:constraints}
\end{align}
where $t_f$ is the time horizon, $\vec{x}_0$ a given initial state, $\Phi(\cdot)$ the final cost and $l(\cdot)$ the intermediate cost function.
$\vec{f(\cdot)}$, $\vec{g}(\cdot)$, and $\vec{h}(\cdot)$ are time-dependent vector fields defining the system dynamics, the equality constraints, and the inequality constraints, respectively.
The problem's associated optimal value function~$V$ (cost-to-go) is defined as
\begin{equation}
  V(t,\vec{x}) = \min\limits_{\substack{\vec{u}(\cdot), \\ \text{s.t.}\, \eqref{eq:constraints}}} \quad \Phi(\vec{x}(t_f)) + \int_t^{t_f} l(\vec{x}(t),\vec{u}(t), t) \de t \; . \label{eq:value_function}
\end{equation}

In principle, our method works with any optimization algorithm that can handle the constraints~\eqref{eq:constraints} and that provides an approximation of the optimal value function~\eqref{eq:value_function}.
The entire solving procedure is denoted mpcSolver in Alg.~\ref{alg:policy_learning}.
In this work, we employ a variant of the \ac{DDP} algorithm called \ac{SLQ} control~\cite{Farshidian17}, which is the continuous-time equivalent to the \ac{iLQR}~\cite{Li04}.
This solver handles the inequality constraints $\vec{h}(\cdot)$ through a barrier function $b(\cdot)$~\cite{Grandia19} and explicitly computes optimal Lagrange multipliers $\vec{\nu}(\cdot)$ for satisfaction of the state-input equality constraint $\vec{g}(\cdot)$ \cite{Farshidian17}.
The Lagrangian of the \ac{OC} problem~\eqref{eq:cost} is therefore given by
\begin{align}
\mathcal{L}(\vec{x}, \vec{u}, t) := l(\vec{x}, \vec{u}, t) &+ \sum_i b\big(\vec{h}_i(\vec{x}, \vec{u}, t)\big) \nonumber \\
&+ \vec{\nu}\transp(t, \vec{x}) \vec{g}(\vec{x}, \vec{u}, t).
\end{align}

The solution of problem~\eqref{eq:cost} consists of nominal state and input trajectories $\{\vec{x}_\text{nom}(\cdot), \vec{u}_\text{nom}(\cdot)\}$ as well as time-dependent linear feedback gains $\mat{K}(t)$ that define the optimal control policy
\begin{equation}
  \vec{\pi}_\text{mpc}(t, \vec{x}) = \vec{u}_\text{nom}(t) + \mat{K}(t)\,(\vec{x} - \vec{x}_\text{nom}(t))  \; . \label{eq:mpc_feedback_law}
\end{equation}
As a byproduct of the solver, we also have access to the state derivative of the value function $\partial_{\vec{x}} V$.

During our emulated real-time MPC loop, we let the solver compute the optimal policy, then store the values of $\{t, \vec{x}, \partial_{\vec{x}} V, \vec{\nu}\}$ at the first time step of the solution in our replay memory. Next, we update the current state using the system dynamics and continue until the rollout length is reached.
\subsection{Policy Loss Function}\label{sec:policy_loss_function}
It is a known property of \ac{OC}~\cite[pp. 111--120]{Bertsekas05} that the optimal input $\vec{u}^*(t)$ must satisfy
\begin{align}
    \vec{u}^*(t, \vec{x}) &= \arg \min\limits_{\vec{u}}  \mathcal{H}(\vec{x},\vec{u},t)
    \; , \label{eq:hamiltonian_minimization} \\
    \mathcal{H}(\vec{x},\vec{u},t) &:= \mathcal{L}(\vec{x},\vec{u},t) + \partial_{\vec{x}} V(t,\vec{x}) \vec{f}(\vec{x},\vec{u}, t) \; , \label{eq:hamiltonian}
\end{align}
where $\mathcal{H}(\cdot)$ is the control Hamiltonian, which directly arises from the \ac{HJB} equation.
Moreover, under some sufficient conditions (so-called Weierstrass conditions), the Hamiltonian attains a strong minimum at $\vec{u}^*$.
Its minimization can, therefore, be seen as a recipe for finding the optimal controls.

A globally optimal policy would have to satisfy~\eqref{eq:hamiltonian_minimization} at \emph{any}~${t, \vec{x}}$.
This minimization, however, has a great drawback, which is commonly referred to as the ``curse of dimensionality".
Furthermore, recording the solution of every time-state pair requires an enormous amount of storage, which is impractical even for moderate-size systems.

We turn to function approximation as a remedy for these difficulties and introduce a parameterized policy $\vec{\pi}(t, \vec{x} | \vec{\theta})$.
Our problem is now to find some parameters $\vec{\theta}^*$ such that $\vec{\pi}(\cdot)$ maps a given $(t,\vec{x})$ pair to a control that achieves a minimum of $\mathcal{H}$.
The substitution of the pointwise optimal $\vec{u}^*$ with the parameterized policy may introduce an optimality gap between the minima $\mathcal{H}(\vec{x}, \vec{u}^*, t)$ and $\mathcal{H}(\vec{x}, \vec{\pi}( \vec{x}, t| \vec{\theta}^*), t)$.
We can relate the size of this optimality gap to the discrepancy in optimal controls through the following lemma:
\revision{
\begin{lemm}\label{lem:hamiltonian_bound}
Given that $\vec{u}^*$ is a strong minimum of equation~\eqref{eq:hamiltonian_minimization} which satisfies the Weierstrass sufficient condition of optimality.
Then the optimality gap in the pointwise minimization of the control Hamiltonian~\eqref{eq:hamiltonian_minimization} upper bounds the distance to the optimal control according to
\begin{equation}\label{eq:hamiltonian_gap_pointwise}
|| \vec{\pi} - \vec{u}^* ||^2
\leq
\frac{2}{\delta} (\mathcal{H}(\vec{x}, \vec{\pi}, t) - \mathcal{H}(\vec{x}, \vec{u}^*, t)) \; ,
\end{equation}%
where $\delta > 0$ pertains to the smallest eigenvalue of $\partial_{\vec{u}}^2 \mathcal{H}$ in the neighborhood of the optimal control, $\vec{u}^*$.
\end{lemm}%
\begin{proof}
    The proof is provided in Appendix~\ref{sec:optimality_proof}.
\end{proof}%
}%
This statement implies that $\vec{\pi}(t, \vec{x} | \vec{\theta}^*)$ is approaching the optimal control for a specific ${(t, \vec{x})}$ pair as the optimality gap in $\mathcal{H}$ is reduced.
It is unrealistic to minimize this gap for every point in state-space simultaneously because that would require an extremely flexible parametrization of $\vec{\pi}$ and would also assume that $\mathcal{H}$ is known at every state.

To our benefit, however, \ac{SLQ} computes the value function along trajectories in state space.
These trajectories induce a distribution over time and state $\{t,\vec{x}\} \sim \mathcal{P}$ that encodes which areas of the time-state-space are visited by an optimal controller.
For our purposes, it is, therefore, sufficient to minimize the optimality gap almost everywhere with respect to $\mathcal{P}$.
Taking the expectation of~\eqref{eq:hamiltonian_gap_pointwise} gives
\begin{equation}
    \E\limits_{\mathcal{P}} \Big[ || \vec{\pi} - \vec{u}^* ||^2 \Big]
    \leq \frac{2}{\delta}
    \E\limits_{\mathcal{P}} \Big[ \mathcal{H}(\vec{x}, \vec{\pi}, t) - \mathcal{H}(\vec{x}, \vec{u}^*, t) \Big] \; ,
\end{equation}
Restricting the minimization only to the relevant states (i.e., those with nonzero probability mass) also allows us to invoke the Universal Approximation Theorem for our parameterized policy:
We can expect to find a $\vec{\theta}^*$ that makes the expected optimality gap in $\mathcal{H}$ arbitrarily small if the function class $\vec{\pi}(t, \vec{x} | \vec{\theta})$ is sufficiently rich.
Our strategy for finding the optimal parameters is, therefore, given by
\begin{equation}
\vec{\theta}^* = \arg \min\limits_{\vec{\theta}} \; \E\limits_{\{t,\vec{x}\} \sim \mathcal{P}} \Big[ \mathcal{H}(\vec{x},\vec{\pi}(t, \vec{x} | \vec{\theta}),t) \Big] \; . \label{eq:policy_loss}
\end{equation}
The quantity in the expectation \eqref{eq:policy_loss} can be seen as a per-sample loss for policy training.
It is essential to realize that the control Hamiltonian allows us to find the optimal control via this unconstrained minimization because the future cost and constraint Lagrangian have already been included.
It is, therefore, not necessary to perform Monte-Carlo-style rollouts to find the optimal control.

The \ac{MPC} loop presented in Sec.~\ref{sec:MPC} serves as a data generation mechanism for the policy search module.
In general terms, the \ac{MPC} fills a replay buffer with data points that correspond to the states that it has encountered, and those tuples are sampled from to compute the empirical expectation in~\eqref{eq:policy_loss}.
In our implementation, the samples for computing the policy gradient are drawn uniformly at random from the replay buffer in order to break their temporal correlation~\cite{Lin92}.
\subsection{Augmenting Samples Using the Optimal Solution}
A favorable property of \ac{SLQ}, being a local dynamic programming approach, is that it computes a second-order approximation of the optimal value function, as well as a first-order approximation of the Lagrange multiplies and control policy, in the vicinity of the optimal state-input trajectories.
In turn, the control Hamiltonian can also be approximated in a region around the optimal solution.
Our numerical investigation in Sec.~\ref{sec:results_loss_function} demonstrates that this approximated Hamiltonian still yields an acceptable estimation of the optimal control input for samples neighboring the MPC trajectories.
This observation motivates us to extract samples not only from the MPC generated trajectories but also from a tube surrounding these trajectories, which further improves the sample complexity of the approach.

By initializing the control problem \eqref{eq:cost}, \eqref{eq:constraints} at feasible, random starting points, the areas where the value function is known corresponds to the subset of states that are visited by a (close-to) optimal policy.
This fact can be exploited to increase the extracted informational content from a rollout of the optimal policy.
By sampling around the nominal state, our data automatically covers tubes in state space, which accelerates learning and makes the learned policy more robust.
This procedure, denoted sampleInNeighborhood in Alg.~\ref{alg:policy_learning}, amounts to drawing states from a Gaussian distribution according to
\begin{equation}
\vec{x} \sim \mathcal{N}(\vec{x}_{\text{nom}}, \Sigma_{\vec{x}}),
\end{equation}
where the covariance matrix has diagonal entries corresponding to the typical disturbance that the respective state component may encounter.
The sampling idea is conceptually similar to fitting the tangent space of the demonstrator policy instead of just the nominal control command~\cite{Mordatch14}.

\subsection{Addressing Distribution Mismatch}
Unfortunately, despite our efforts to extract samples from trajectories that cover a large volume in state space, there is still a bias of the state distribution towards those states that are encountered by the optimal policy.
This distribution mismatch is a common problem in \ac{IL} and stems from the fact that a learned controller produces inevitably different control inputs than the demonstrator (even when fully converged, unstable physical systems may amplify small differences), which will eventually drive the system into an area of the state space from which no data is available.
To avoid this scenario, we use a behavioral policy $\vec{\pi}_{\mathcal{B}}(\cdot)$ to push the emulated \ac{MPC} loop towards the states that will be seen by the learned policy.
Taking inspiration from Dagger~\cite{Ross11}, the update rule for the next state (stepSystem method in Alg.~\ref{alg:policy_learning}) is given by
\begin{align}
\vec{x}(t + \Delta t) &= \vec{x}(t) + \vec{f}(t, \vec{x}, \vec{\pi}_{\mathcal{B}}(t, \vec{x}, \vec{\theta})) \Delta t , \\
\vec{\pi}_{\mathcal{B}}(t, \vec{x}, \vec{\theta}) &= (1-\alpha) \vec{\pi}_{\text{mpc}} + \alpha \vec{\pi}(t, \vec{x} | \vec{\theta}) ,
\end{align}
where the mixing parameter $\alpha$ is initially zero and linearly increases with the number of iterations until it has reached one in the final iteration.
Through this process, the learned policy is gradually given more responsibility to decide where the \ac{MPC} algorithm should be applied.
It is important to note that the \ac{MPC} solver is not influenced by the learned policy and produces optimal solutions independent of the value of~$\alpha$.
\subsection{Policy Structure and Training}
Now that the loss function and a way to populate our experience buffer is defined, we turn the actual training procedure and computation of stochastic gradients of our policy.

In this work, we use a mixture-of-experts architecture~\cite{Jacobs91} for the control policy, shown in Fig.~\ref{fig:expert_mix_net}.
Allowing multiple policies $\vec{\pi}_i$ to compete naturally handles the non-uniqueness of the \ac{OC} solution.
For example, passing an obstacle around the left or right side may be an equally good choice that two different experts will try to imitate, but forcing a monolithic network to interpolate between these solutions can be catastrophic.
\begin{figure}
    \centering
    \includegraphics[width=0.82\columnwidth]{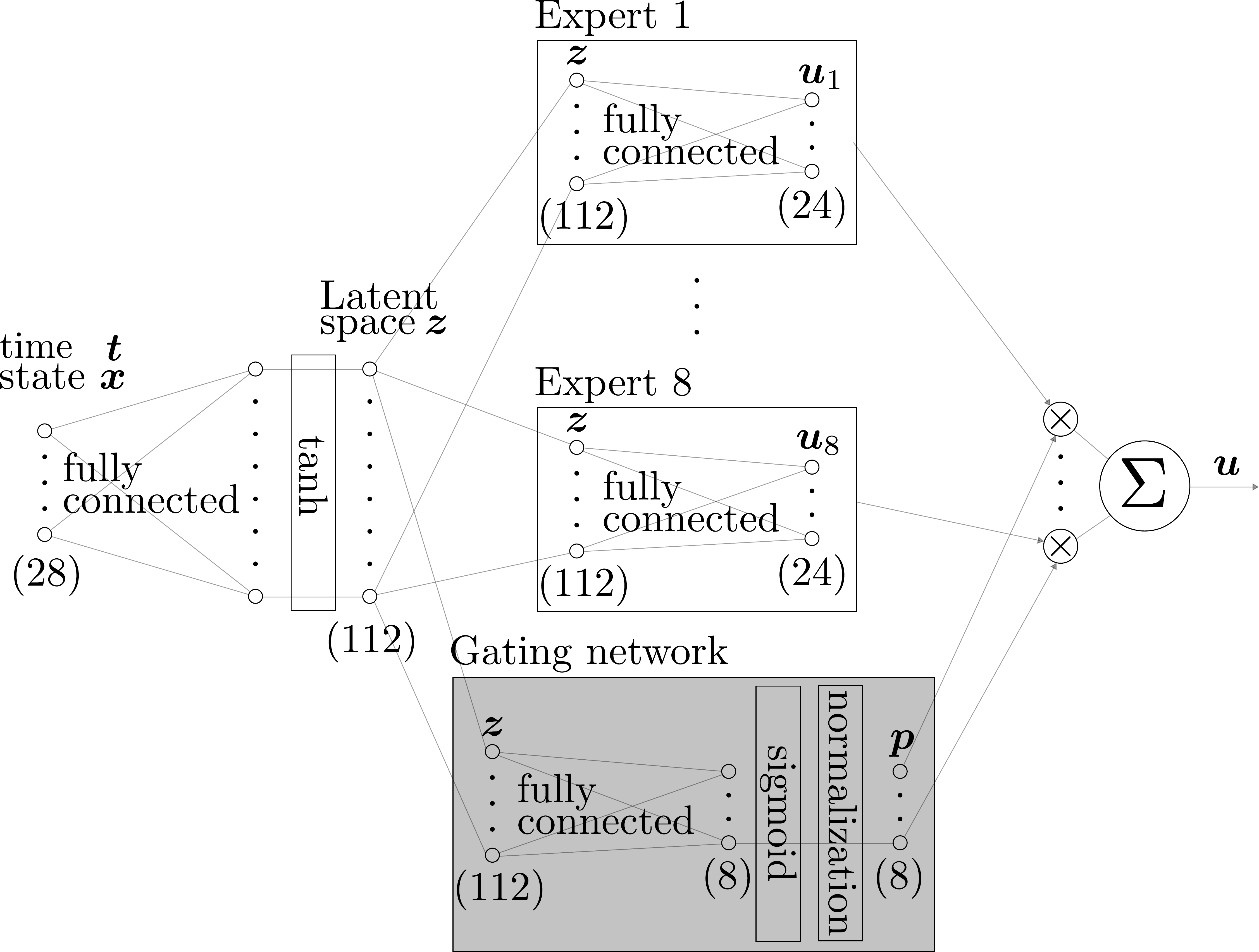}
    \caption{Architecture of our mixture-of-experts network. The dimensions correspond to the instantiation for the ANYmal robot.}
    \label{fig:expert_mix_net}
\end{figure}

The final control output of the network is a convex combination of the outputs of different expert sub-policies
\begin{equation}
\vec{\pi}(t, \vec{x} | \vec{\theta}) = \sum\limits_{i=1}^{\text{NumExperts}} p_i(t, \vec{x} | \vec{\theta}) \, \vec{\pi}_i(t, \vec{x} | \vec{\theta}) \; . \label{eq:expert_net}
\end{equation}
The mixing coefficients $p_i$ are the output of a gating network whose final activation ensures that all coefficients are positive and sum up to one.
While a softmax layer achieves this constraint, we find that a sigmoid activation with subsequent normalization performs better in selecting a consistent number of experts for a given task across multiple training runs.
We believe the reason for this observation is that the softmax activation is too sharp in selecting one specific expert such that an unlucky initialization may lead to some experts never even being considered and therefore not receiving policy updates.

Both the expert sub-policies and the gating network share a common latent space representation.
The overall policy~\eqref{eq:expert_net} remains a feed-forward neural network and can, therefore, be trained with standard deep learning optimization techniques:
At each policy iteration step, we draw a batch of $\{t, \vec{x},\partial_{\vec{x}} V, \vec{\nu}\}$ tuples from the replay buffer and compute the empirical loss for this batch as
\begin{align}
\text{loss} =  \sum\limits_{j=1}^{\text{BatchSize}} \sum\limits_{i=1}^{\text{NumExperts}} p_i(t_j, \vec{x}_j | \vec{\theta})  \mathcal{H}(\vec{x}_j,\vec{\pi}_i(t_j, \vec{x}_j | \vec{\theta}),t_j) \; . \label{eq:empirical_batch_loss}
\end{align}
Note that we force each experts' output to individually minimize the Hamiltonian to encourage specialization~\cite{Jacobs91}.
This procedure is slightly different from inserting~\eqref{eq:expert_net} into~\eqref{eq:policy_loss}, which would only encourage their combined output to be optimal.
Training the optimal policy involves taking gradient steps in the parameter space.
The policy gradient for the loss function~\eqref{eq:empirical_batch_loss} for a given sample $j$ is equal to
\begin{align}
\sum\limits_{i=1}^{N_{\text{experts}}}
    &p_i(t_j, \vec{x}_j | \vec{\theta}) \partial_{\vec{u}} \mathcal{H}(\vec{x}_j,\vec{\pi}_i(t_j, \vec{x}_j | \vec{\theta}),t_j) \partial_{\vec{\theta}}\vec{\pi}(t_j, \vec{x}_{i} | \vec{\theta})
    \nonumber \\
    &+ \partial_{\vec{\theta}} p_i(t_j, \vec{x}_j | \vec{\theta}) \mathcal{H}(\vec{x}_j,\vec{\pi}_i(t_j, \vec{x}_j | \vec{\theta}),t_j) .
\end{align}
For all nominal states the control derivative of the Hamiltonian $\partial_{\vec{u}} \mathcal{H}$ is computed as a byproduct of solving the problem \eqref{eq:cost}, \eqref{eq:constraints};
for neighboring states the derivative of~\eqref{eq:hamiltonian} can be queried.
The gradients of $p$ and $\vec{\pi}$ are calculated by backpropagation.
%
%
\section{Results}
We assess the policy structure and loss function of the MPC-Net algorithm separately to highlight the performance of our method and justify individual design choices.
\subsection{Experimental Setup}
The results presented in this document are produced with the quadrupedal robot ANYmal (Fig.~\ref{fig:anymal}), which is an example of a hybrid system with time-varying flow map and constraints.
The constraints encode zero contact forces for a foot in swing phase and zero velocity when in stance phase.
\begin{figure}
    \centering
    \includegraphics[width=0.5\columnwidth]{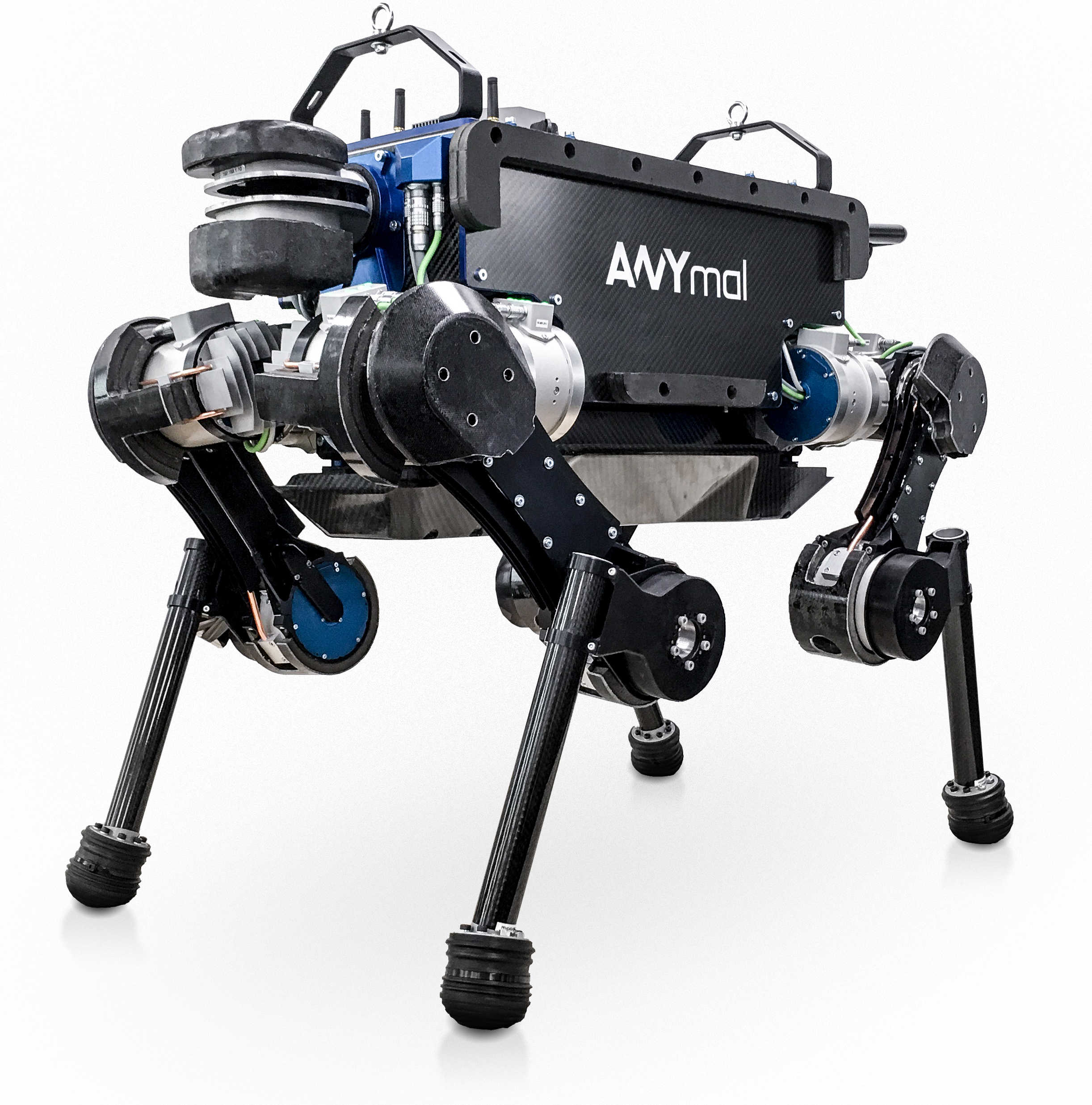}
    \caption{The quadrupedal robot ANYmal. The floating base and three joints per leg amount to 18 \ac{DOF}. Our kinodynamic model of this robot has 24 states and 24 inputs.}
    \label{fig:anymal}
\end{figure}
Our kinodynamic model amounts to 24 states (base pose, base twist, joint angles) and 24 control inputs (joint velocities, foot contact forces).
The control commands from our policy are fed to a whole-body tracking controller that computes the final actuator torque commands.
Instead of providing the absolute time to the network, it is more expedient to encode the phase of the gait cycle of the legged robot.
By abuse of notation, we, therefore, define four `time' variables, one per leg, which are zero during stance phases and describe half a period of a sine wave during the swing motion.

We use a quadratic \ac{OC} cost function~\eqref{eq:cost} of the form
\begin{align}
\Phi(\vec{x}) &= (\vec{x} - \vec{x}_{\text{ref,}f})^\top \mat{Q}_f (\vec{x} - \vec{x}_{\text{ref,}f}) \; , \\
l(\vec{x},\vec{u},t) &= (\vec{x}-\vec{x}_\text{ref}(t))^\top \mat{Q}\, (\vec{x}-\vec{x}_\text{ref}(t)) + \vec{u}^\top \mat{R} \vec{u} .
\end{align}
The reference states encourage the system to return to the origin with a trotting or static walk gait and then maintain a nominal configuration.
Our quadratic cost structure, together with the fact that our constraints and dynamics are input-affine, makes the Hamiltonian a quadratic function in $\vec{u}$.
Notably, this guarantees that all assumptions made for the proof of Lemma~\ref{lem:hamiltonian_bound} are fulfilled.

Since our loss function~\eqref{eq:empirical_batch_loss} directly depends on the sampled data, it is not a suitable termination criterion for the training process and has a high variance.
We monitor the training progress of our policy by computing a rollout of the system dynamics $\vec{f}(\cdot)$ with the learned policy from random initial points.
A rollout lasts \unit[3]{s} but is terminated early if the pitch or roll angle exceed \unit[30]{$^\circ$} or the height deviates more than \unit[20]{cm} from the nominal value.
This procedure can be seen as a test set for our learning approach.
The resulting average rollout cost~\eqref{eq:cost} and the survival time are good performance indicators for the policy.

All hyper-parameters of our algorithm are summarized in Tab.~\ref{tab:hyperparameters}.
The network weights are randomly initialized before training and optimized with the AMSGrad variant of the Adam optimizer~\cite{Kingma14, Reddi18}, which implements the stepOptimizer primive in Alg.~\ref{alg:policy_learning}.
We take the data from \ac{MPC} as is without any pruning of failed rollouts or outlier states.
For the following comparisons, we execute five training runs for each configuration and average the results.
For better interpretability, the progression of training is shown in terms of the total duration of accumulated rollouts rather than (linearly related) optimizer iterations.
\begin{table}
    \centering
    \caption{Hyperparameters of MPC-Net}
    \begin{tabularx}{0.9\columnwidth}{rl||rl}
        \toprule
        maxIter& 100'000 & mpcDecimation & 500\\
        rolloutLength & \unit[3]{s} & Replay Buffer Size & 100'000\\
       time step $\Delta t$ & \unit[0.0025]{s} & $N_\text{experts}$ & 8\\
        learningRate & 1e-3 & batchSize & 32 \\
        \bottomrule
    \end{tabularx}
    \label{tab:hyperparameters}
\end{table}
\subsection{Loss Function}\label{sec:results_loss_function}
We begin by providing numerical evidence that minimizing the control Hamiltonian $\mathcal{H}$ yields optimal controls.
To this end, we compare the optimal policy from \ac{MPC}~\eqref{eq:mpc_feedback_law} with the result of the minimization~\eqref{eq:hamiltonian_minimization}.
Table~\ref{tab:hamiltonian_minimization} shows median constraint violation and relative deviation from the optimal input for 40 randomly drawn points on or near optimal trajectories.
For query states~$\tilde{\vec{x}}$ near the optimal trajectory the benchmark control $\vec{u}^*$ is computed by solving~\eqref{eq:cost} for $\vec{x}_0 = \tilde{\vec{x}}$.
The values confirm that our estimation of $\mathcal{H}$ is sufficiently accurate and that its minimization produces constraint-satisfactory control commands.
Additionally, this result suggests that the requirement for Lemma~\ref{lem:hamiltonian_bound} (i.e., the existence of a strong minimum) also holds for states in the vicinity of an optimal trajectory.
\begin{table}
    \centering
    \caption{Comparison between the MPC policy and Hamiltonian minimization. We show median constraint violation $||\vec{g}||$ and the relative error to the optimal control $\vec{u}^*$}
    \begin{tabularx}{0.9\columnwidth}{r|ll|ll}
         & \multicolumn{2}{c}{states on opt. trajectory} & \multicolumn{2}{c}{states near opt. trajectory} \\
         &  $||\vec{g}||$ & $\frac{||\vec{u} - \vec{u}^*||}{||\vec{u}||}$ & $||\vec{g}||$ & $\frac{||\vec{u} - \vec{u}^*||}{||\vec{u}||}$ \\
        \toprule
       $\vec{\pi}_\text{mpc}$ & 3.44e-6 & 0.0 & 3.58e-4 & 2.48e-2 \\
       $\arg \min \mathcal{H}$ & 3.46e-4 & 1.58e-3 & 5.40e-4 & 2.80e-2 \\
        \bottomrule
    \end{tabularx}
    \label{tab:hamiltonian_minimization}
\end{table}

\subsection{Comparison to Behavior Cloning}
The next experiment compares our proposed Hamiltonian~\eqref{eq:policy_loss} as a loss function with a simpler \ac{BC} loss that encourages matching of the demonstrator's control command
\begin{equation}
\vec{\theta}^* = \arg \min\limits_{\vec{\theta}} \E\limits_{\{t,\vec{x}\} \sim \mathcal{P}} || \vec{\pi}_{\text{mpc}}(t, \vec{x}) - \vec{\pi}(t, \vec{x} | \vec{\theta}) ||_{\mat{R}} \; . \label{eq:bc_loss}
\end{equation}
We use the control cost matrix $\mat{R}$ here to normalize the different control dimensions.
We see in Fig.~\ref{fig:hamiltonian_vs_imitation_learning} that the simpler loss \eqref{eq:bc_loss} results in similar convergence to a stable control law, but the Hamiltonian loss consistently achieves a lower constraint violation value.
Constraint violation means that physical feasibility is violated, effectively allowing the robot to stabilize by cheating.
When deployed in a physics simulator, the policy trained on \eqref{eq:bc_loss} tends to fall after a few footsteps as violations errors accumulate because cheating is not possible anymore.

We conjecture that the structure of the Hamiltonian, which includes constraint violation penalties explicitly, encourages the learning algorithm to respect constraints more carefully than in the case of only observing constraint-consistent demonstrations.
Note that our loss would inform the learner about constraint violations even if the demonstrations violated them.
\begin{figure}
    \centering
    \includegraphics[width=\columnwidth]{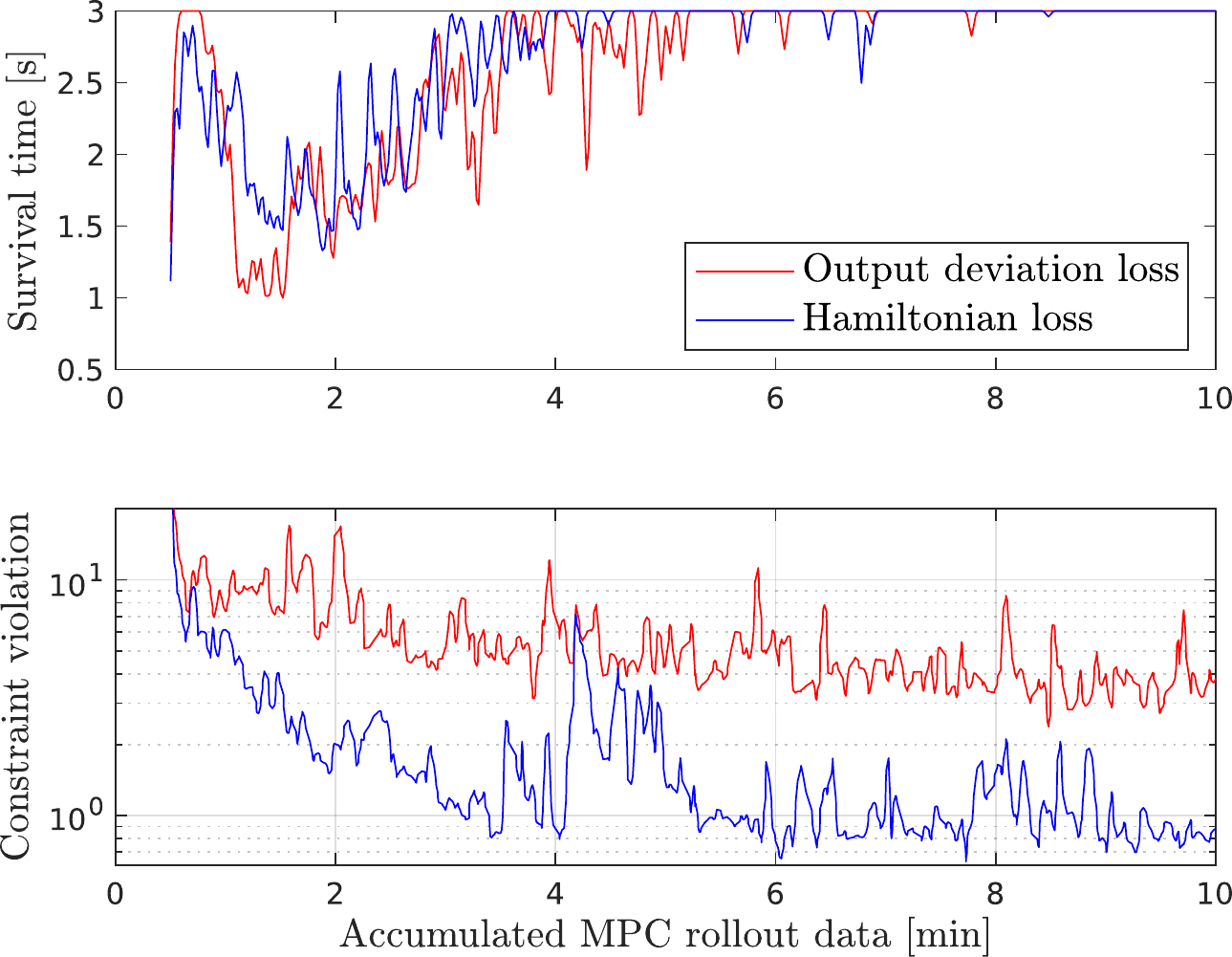}
    \caption{Comparison between minimization of the control Hamiltonian and a simpler loss penalizing differences in policy outputs.
    Both loss functions are applied to the mixture-of-expert network architecture.}
    \label{fig:hamiltonian_vs_imitation_learning}
\end{figure}
\subsection{Sample Efficiency}
We show in Fig.~\ref{fig:sampling_around_trajectory} how sampling around the optimal trajectory influences the learning process for a quadruped walking motion.
There is no noticeable effect in the loss function (i.e., the value of the Hamiltonian) throughout the process, which also suggests that this value is not a good indicator for the actual performance of the policy.
Instead, a clear effect can be seen in the progression of the survival time.
The plot suggests that the additionally sampled states provide valuable information for the training algorithm to learn faster and stabilize the system more consistently at the end of the training.
More importantly even, we observe that the policy that is trained only on nominal samples is overly aggressive to small deviations in the system's state.
These strong gains lead to oscillatory behavior when deployed on the real system, where sensors and the state estimator inevitably introduce noise.
Subsequently, only the policy that is trained with additional samples around the optimal trajectory is robust and smooth enough to stabilize the system under noisy state estimates.
Evidence of this result is shown in the video\footnote{\texttt{https://youtu.be/VI7wt5PCJ14}}.

Finally, experiments show that the policies with sampling become usable on the robot at approximately 75\% of the maximum number of iterations, indicating that sampling also improves the effective amount of information extracted from demonstrated trajectories and thereby necessitating fewer \ac{MPC} calls.
Our algorithm, therefore, learns to stabilize a walking robot from an experience buffer that is equivalent to running the robot for nine minutes with an optimal controller.
Notably, this time scale opens up the possibility of learning directly on a real system.
\begin{figure}
    \centering
    \includegraphics[width=0.95\columnwidth]{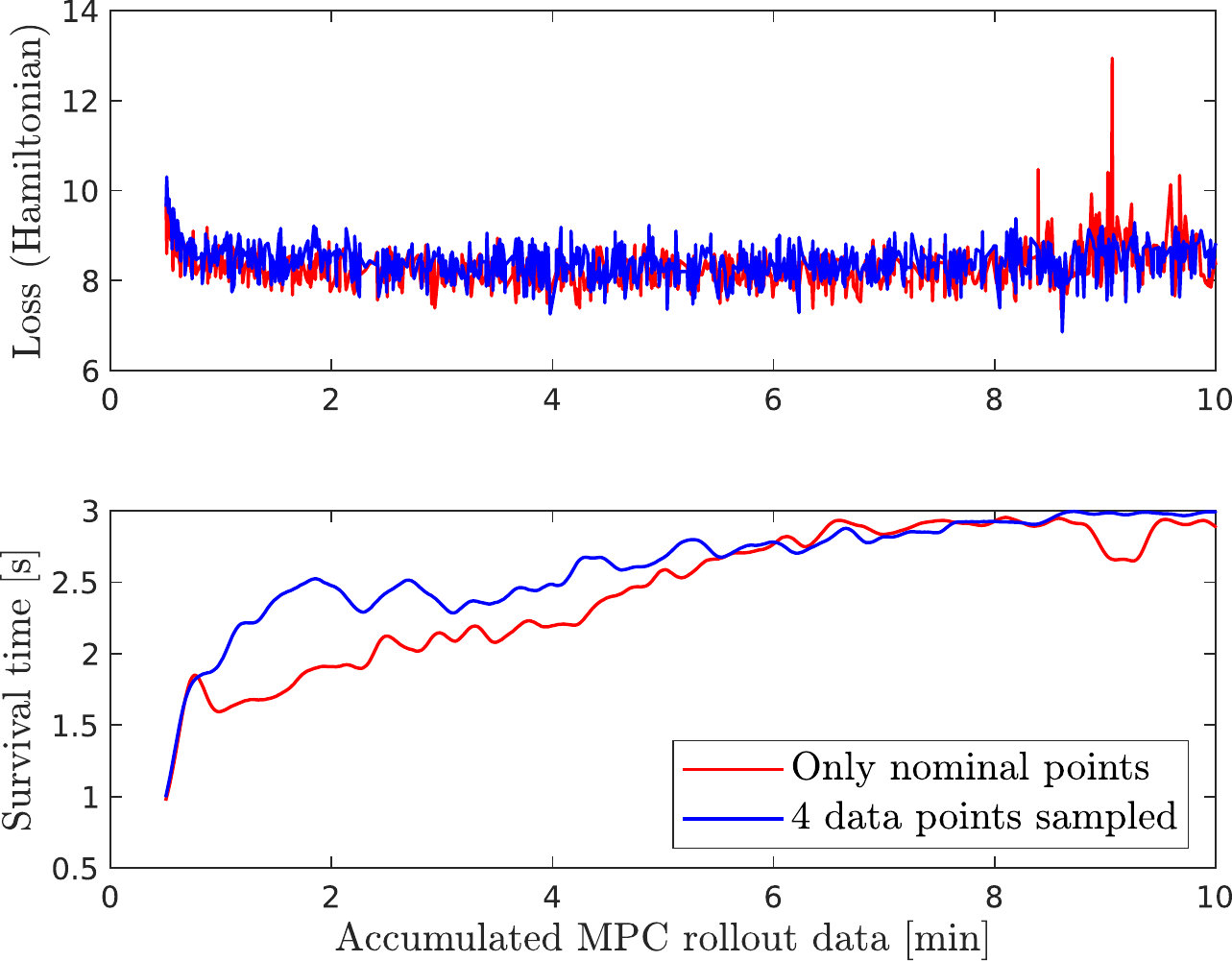}
    \caption{Effect of collecting additional samples around the optimal trajectory. The maximum duration of a policy rollout is \unit[3]{s}. Five independent experiments are averaged for each setting.}
    \label{fig:sampling_around_trajectory}
\end{figure}
\subsection{Mixture-of-Expert Architecture}
In this experiment we compare the performance of our mixture-of-expert architecture to a classical two-layer \ac{MLP}
\begin{equation}
\vec{\pi}_{\text{MLP}} = \mat{A}_2(\tanh(\mat{A}_1 \vec{x} + \vec{b}_1)) + \vec{b}_2 \; ,
\end{equation}
with an equally-sized latent space than the one of the expert mixture.\footnote{
We also tested deeper and wider \ac{MLP} architectures but could not observe improved performance.} While both architectures achieve similar convergence to a stable controller,
Fig.~\ref{fig:expert_net} shows that the expert mixture reaches a significantly better constraint violation score.

We allow the expert mixture network to use 8 experts for training.
Interestingly, the gating network decides to use fewer experts, and swiching between these sub-policies happens precisely at the times when the contact configuration of the system changes.
For a trotting gait, only three experts are needed (blue expert for the first pair of diagonal legs, a mixture of red and black for the other pair, and red for the final stance phase) while a static walk selects four experts, one per swing leg.
This result shows that the policy learns to select an appropriate expert in different domains of the state space.
Moreover, a specialized expert that focuses only on a specific contact configuration learns to obey the constraints better than a single policy for all phases of the gait.
\begin{figure}
    \centering
    \includegraphics[width=\columnwidth]{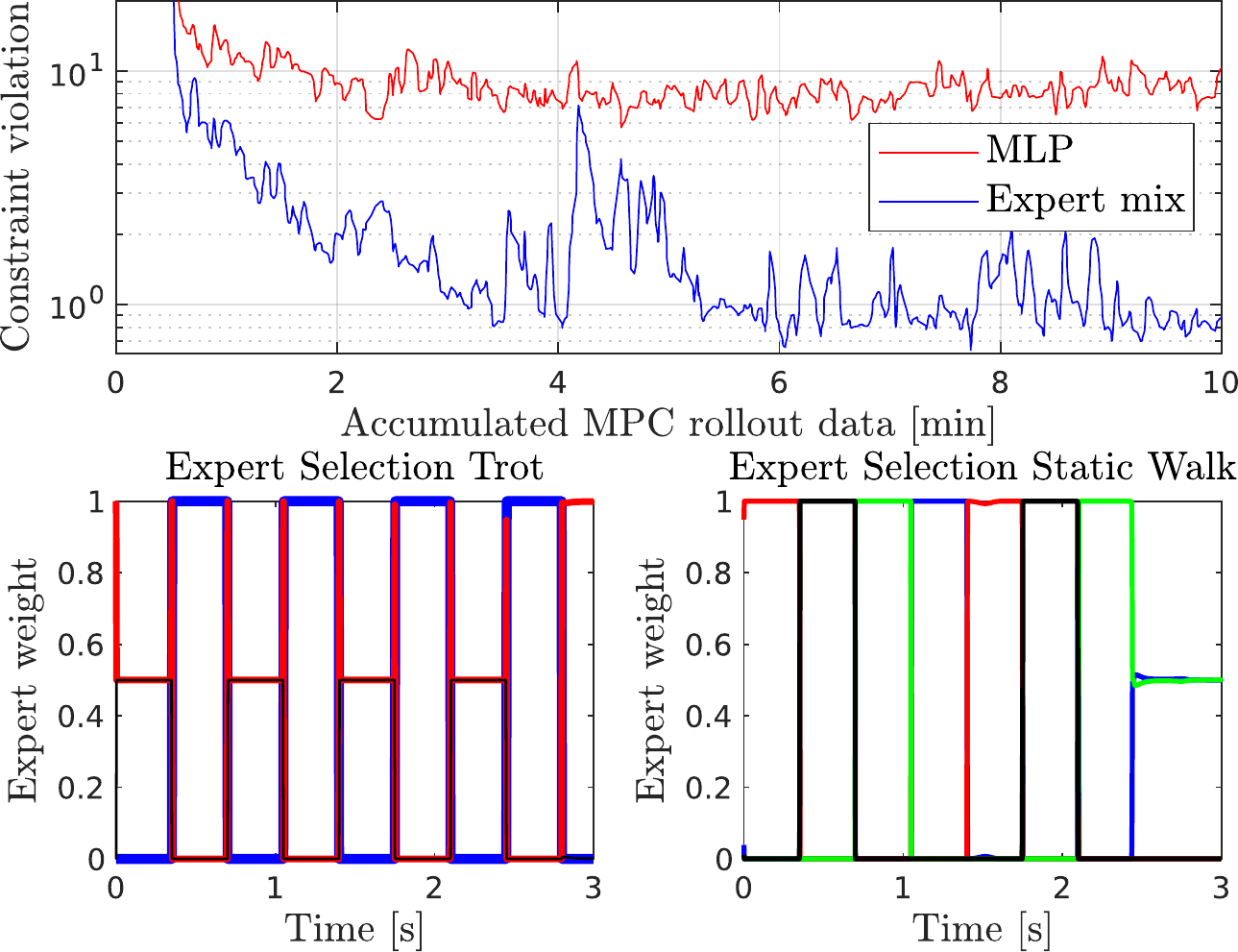}
    \caption{The top graph shows a comparison of constraint violation during training between the expert mixture network and a \ac{MLP} of equivalent size. The bottom two graphs display the output of the expert gating network for two different gaits (one color per expert). Switching times correspond exactly to changes in the contact configuration and the pattern repeats periodically with the period of the gait.}
    \label{fig:expert_net}
\end{figure}
\subsection{Robot Control}
Finally, we test our trained policy on the physical ANYmal robot.
The on-board policy evaluation takes approx. \unit[0.125]{ms}, compared to \unit[38]{ms} of an MPC update, and can therefore be called synchronously to the tracking controller.
We verify that both a trotting and a static walk gait can be learned from the \ac{MPC} oracle using the same network structure and identical hyperparameters.
Despite the seemingly more stable static walk, both gaits pose a comparable level of difficulty to the learning algorithm which manifests in similar convergence properties.
The attached video shows the robot's behavior under our learned policy.

We test the policy's ability to return to the origin by starting the robot at a nonzero initial displacement and yaw rotation.
In Figure~\ref{fig:anymal_return_zero}, we plot the resulting state trajectories of $x$-$y$ position as well as yaw angle, confirming that the network succeeds in the regularization task without overshoot.
\begin{figure}
    \centering
    \includegraphics[width=0.95\columnwidth]{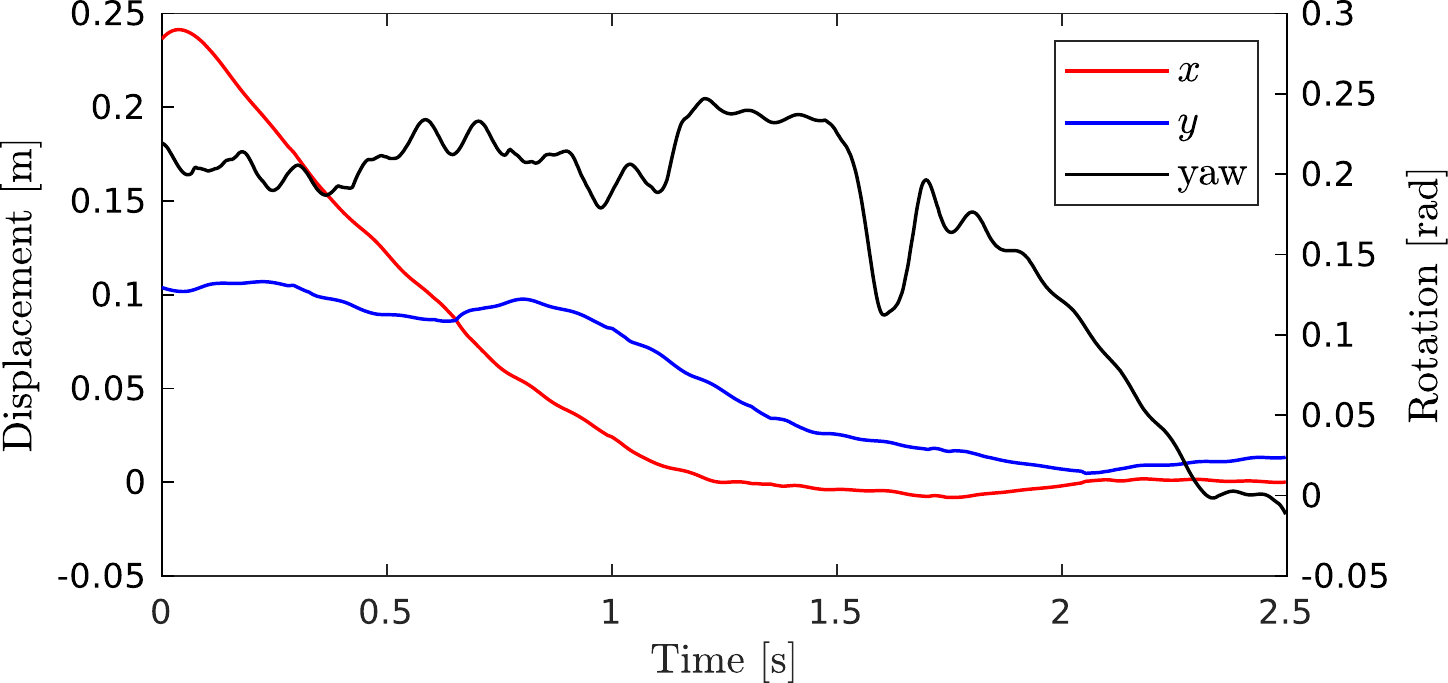}
    \caption{Time evolution of ANYmal's base position and yaw angle under the trained policy. All quantities return to zero with minimal overshoot.}
    \label{fig:anymal_return_zero}
\end{figure}
%
%
\section{Conclusion}
In this work, we explored a variant of \ac{MPC}-guided policy search to learn a feedback control law.
Contrary to other imitation learning approaches, which try to mimic the control commands of a teacher, our formulation is based on minimizing the control Hamiltonian.
The optimization corresponds to solving the \ac{OC} problem with a restricted family of control laws.
We show that our algorithm is capable of learning a feedback policy for two different gaits of a walking robot from less than 10 minutes of demonstration data.

By design, our method cannot outperform the \ac{MPC} policy because it optimizes the same cost function, and we cannot learn in areas where the optimization algorithm does not converge.
However, the improved speed in control evaluation may very well stabilize motions that were not possible before or enable online control altogether.
Even if the \ac{MPC} algorithm is too slow to stabilize the robot, the sample efficiency of our methods facilitates learning directly on the hardware.
To this end, one may compute the MPC solution online on a powerful, off-board machine.

Online \ac{MPC} may also be too energy-consuming for longer autonomous operation, in which case MPC-Net could control the robot by default and \ac{MPC} is only queried as soon as unknown states are encountered.
Such an operating mode would ensure the safety of the system while generating more training data in regions of the state-space that are still uncertain.

A limitation related to imitating optimized trajectories is the lack of exploration, as our policy search method will fall into the same local minima that the \ac{MPC} optimizer found.
Future research is necessary to investigate how policies could systematically request new samples from the \ac{MPC} to improve in areas where the optimal control is still uncertain.
%
\revision{
\section{Appendix: Proof to Lemma~\ref{lem:hamiltonian_bound}}\label{sec:optimality_proof}
Let $\vec{p} := \vec{\pi} - \vec{u}^*$ be the difference between an arbitrary policy $\vec{\pi}$ and the optimal controls $\vec{u}^*$ for a single time and state.
We define $\bar{\mathcal{H}}: [0, 1] \rightarrow \mathbb{R}$ as
\begin{equation}
\bar{\mathcal{H}}(\alpha) = \mathcal{H}(\vec{x}, \vec{u}^* + \alpha \vec{p}, t) \; ,
\end{equation}
where $\alpha$ is an interpolation parameter.
Assuming second order differentiability of $\mathcal{H}$,
the Fundamental Theorem of Calculus allows us to write
\begin{equation} \label{eq:FTC_Hamiltonian_1}
\bar{\mathcal{H}}(1) - \bar{\mathcal{H}}(0) = \int_{0}^{1} \partial_{\vec{u}} \bar{\mathcal{H}}(\alpha) \, \vec{p}  \de \alpha \;.
\end{equation}
Applying the same idea again to the integrand, we get
\begin{equation} \label{eq:FTC_Hamiltonian_2}
\partial_{\vec{u}} \bar{\mathcal{H}}(\alpha) - \underbrace{\partial_{\vec{u}} \bar{\mathcal{H}}(0)}_{= \mathbf{0}}
= \int_{0}^{\alpha} \vec{p}\transp \partial_{\vec{u}}^2 \bar{\mathcal{H}}(\beta) \de \beta \; ,
\end{equation}
with $\partial_{\vec{u}}^2 \bar{\mathcal{H}}$ denoting the Hessian matrix w.r.t. $\vec{u}$.
The second term above vanishes because $\vec{u}^*$ is optimal.
Now we can substitute~\eqref{eq:FTC_Hamiltonian_2} into~\eqref{eq:FTC_Hamiltonian_1}, resulting in
\begin{equation}
\bar{\mathcal{H}}(1) - \bar{\mathcal{H}}(0) = \vec{p}\transp
\left(\int_{0}^{1} \!\! \int_{0}^{\alpha}  \partial_{\vec{u}}^2 \bar{\mathcal{H}}(\beta) \de \beta \de \alpha \right) \vec{p} \; .
\end{equation}
%
Assuming that $\vec{u}^*$ is a strong minimum which satisfies the Weierstrass sufficient condition, the Hessian of the Hamiltonian is positive definite in the neighborhood of the optimal input.
Thus for small enough $\vec{p}$, there exists a positive scalar $\delta > 0$ for which
\begin{equation}
\partial_{\vec{u}}^2 \bar{\mathcal{H}}(\beta) > \delta \vec{I}, \text{ for all $\beta \in [0, 1]$}
\end{equation}
We then have,
\begin{equation}
\vec{p}\transp
\left(\int_{0}^{1} \!\! \int_{0}^{\alpha}  \partial_{\vec{u}}^2 \bar{\mathcal{H}}(\beta) \de \beta \de \alpha \right) \vec{p} > \frac{\delta}{2} \, \vec{p}\transp\vec{p} \, .
\end{equation}
Replacing the left side of the inequality with ${\bar{\mathcal{H}}(1) - \bar{\mathcal{H}}(0)}$ and recalling the definition of $\vec{p}$ yields the statement of Lemma~\ref{lem:hamiltonian_bound}.
\qed
}

%
%
\bibliographystyle{IEEEtran}
\bibliography{sources} 

\begin{thebibliography}{10}
\providecommand{\url}[1]{#1}
\csname url@samestyle\endcsname
\providecommand{\newblock}{\relax}
\providecommand{\bibinfo}[2]{#2}
\providecommand{\BIBentrySTDinterwordspacing}{\spaceskip=0pt\relax}
\providecommand{\BIBentryALTinterwordstretchfactor}{4}
\providecommand{\BIBentryALTinterwordspacing}{\spaceskip=\fontdimen2\font plus
\BIBentryALTinterwordstretchfactor\fontdimen3\font minus
  \fontdimen4\font\relax}
\providecommand{\BIBforeignlanguage}[2]{{%
\expandafter\ifx\csname l@#1\endcsname\relax
\typeout{** WARNING: IEEEtran.bst: No hyphenation pattern has been}%
\typeout{** loaded for the language `#1'. Using the pattern for}%
\typeout{** the default language instead.}%
\else
\language=\csname l@#1\endcsname
\fi
#2}}
\providecommand{\BIBdecl}{\relax}
\BIBdecl

\bibitem{Tan18}
J.~Tan, T.~Zhang, E.~Coumans, A.~Iscen, Y.~Bai, D.~Hafner, S.~Bohez, and
  V.~Vanhoucke, ``Sim-to-real: Learning agile locomotion for quadruped
  robots,'' in \emph{Robotics: Science and Systems XIV}, 2018.

\bibitem{Iscen18}
A.~Iscen, K.~Caluwaerts, J.~Tan, T.~Zhang, E.~Coumans, V.~Sindhwani, and
  V.~Vanhoucke, ``Policies modulating trajectory generators,'' in \emph{Conf.
  on Robot Learning (CoRL)}, 2018, pp. 916--926.

\bibitem{Haarnoja18}
T.~Haarnoja, S.~Ha, A.~Zhou, J.~Tan, G.~Tucker, and S.~Levine, ``Learning to
  walk via deep reinforcement learning,'' in \emph{Robotics: Science and
  Systems XV}, 2019.

\bibitem{Hwangbo19}
J.~Hwangbo, J.~Lee, A.~Dosovitskiy, D.~Bellicoso, V.~Tsounis, V.~Koltun, and
  M.~Hutter, ``Learning agile and dynamic motor skills for legged robots,''
  \emph{Science Robotics}, vol.~4, no.~26, 2019.

\bibitem{Xie19}
Z.~Xie, P.~Clary, J.~Dao, P.~Morais, J.~W. Hurst, and M.~van~de Panne,
  ``Iterative reinforcement learning based design of dynamic locomotion skills
  for cassie,'' \emph{CoRR}, vol. abs/1903.09537, 2019.

\bibitem{Osa18}
T.~Osa, J.~Pajarinen, G.~Neumann, J.~A. Bagnell, P.~Abbeel, and J.~Peters, ``An
  algorithmic perspective on imitation learning,'' \emph{Foundations and Trends
  in Robotics}, vol.~7, no. 1-2, pp. 1--179, 2018.

\bibitem{Sun17}
W.~Sun, A.~Venkatraman, G.~J. Gordon, B.~Boots, and J.~A. Bagnell, ``Deeply
  aggrevated: Differentiable imitation learning for sequential prediction,'' in
  \emph{Int. Conf. on Machine Learning {ICML}}, 2017, pp. 3309--3318.

\bibitem{Park15}
H.~Park, P.~M. Wensing, and S.~Kim, ``Online planning for autonomous running
  jumps over obstacles in high-speed quadrupeds,'' in \emph{Robotics: Science
  and Systems XI}, 2015.

\bibitem{Naveau17}
M.~{Naveau}, M.~{Kudruss}, O.~{Stasse}, C.~{Kirches}, K.~{Mombaur}, and
  P.~{Souères}, ``A reactive walking pattern generator based on nonlinear
  model predictive control,'' \emph{IEEE Robotics and Automation Letters},
  vol.~2, no.~1, pp. 10--17, 2017.

\bibitem{Farshidian17MPC}
F.~{Farshidian}, E.~{Jelavic}, A.~{Satapathy}, M.~{Giftthaler}, and
  J.~{Buchli}, ``Real-time motion planning of legged robots: A model predictive
  control approach,'' in \emph{IEEE-RAS Int. Conf. on Humanoid Robotics
  (Humanoids)}, 2017, pp. 577--584.

\bibitem{Winkler18}
A.~W. {Winkler}, C.~D. {Bellicoso}, M.~{Hutter}, and J.~{Buchli}, ``Gait and
  trajectory optimization for legged systems through phase-based end-effector
  parameterization,'' \emph{IEEE Robotics and Automation Letters}, vol.~3,
  no.~3, pp. 1560--1567, 2018.

\bibitem{Neunert18}
M.~{Neunert}, M.~{Stäuble}, M.~{Giftthaler}, C.~D. {Bellicoso}, J.~{Carius},
  C.~{Gehring}, M.~{Hutter}, and J.~{Buchli}, ``Whole-body nonlinear model
  predictive control through contacts for quadrupeds,'' \emph{IEEE Robotics and
  Automation Letters}, vol.~3, no.~3, pp. 1458--1465, 2018.

\bibitem{Carius19}
J.~Carius, R.~Ranftl, V.~Koltun, and M.~Hutter, ``Trajectory optimization for
  legged robots with slipping motions,'' \emph{{IEEE} Robotics and Automation
  Letters}, vol.~4, no.~3, pp. 3013--3020, 2019.

\bibitem{Ratliff06}
N.~D. Ratliff, D.~M. Bradley, J.~A. Bagnell, and J.~E. Chestnutt, ``Boosting
  structured prediction for imitation learning,'' in \emph{Advances in Neural
  Information Processing Systems}, 2006, pp. 1153--1160.

\bibitem{Abbeel10}
P.~Abbeel, A.~Coates, and A.~Y. Ng, ``Autonomous helicopter aerobatics through
  apprenticeship learning,'' \emph{Int. J. Robotics Res.}, vol.~29, no.~13, pp.
  1608--1639, 2010.

\bibitem{Mordatch14}
I.~Mordatch and E.~Todorov, ``Combining the benefits of function approximation
  and trajectory optimization,'' in \emph{Robotics: Science and Systems X},
  2014.

\bibitem{Levine13}
S.~Levine and V.~Koltun, ``Guided policy search,'' in \emph{Int. Conf. on
  Machine Learning {ICML}}, 2013, pp. 1--9.

\bibitem{Levine14}
------, ``Learning complex neural network policies with trajectory
  optimization,'' in \emph{Int. Conf. on Machine Learning {ICML}}, 2014, pp.
  829--837.

\bibitem{Kahn17}
G.~Kahn, T.~Zhang, S.~Levine, and P.~Abbeel, ``{PLATO:} policy learning using
  adaptive trajectory optimization,'' in \emph{IEEE Int. Conf. on Robotics and
  Automation {ICRA}}, 2017, pp. 3342--3349.

\bibitem{Choudhury17}
S.~Choudhury, A.~Kapoor, G.~Ranade, S.~Scherer, and D.~Dey, ``Adaptive
  information gathering via imitation learning,'' in \emph{Robotics: Science
  and Systems XIII}, 2017.

\bibitem{Yang19}
Y.~Yang, K.~Caluwaerts, A.~Iscen, T.~Zhang, J.~Tan, and V.~Sindhwani, ``Data
  efficient reinforcement learning for legged robots,'' \emph{CoRR}, vol.
  abs/1907.03613, 2019.

\bibitem{Atkeson02}
C.~G. Atkeson and J.~Morimoto, ``Nonparametric representation of policies and
  value functions: {A} trajectory-based approach,'' in \emph{Advances in Neural
  Information Processing Systems {NIPS}}, 2002, pp. 1611--1618.

\bibitem{Zhong13}
M.~Zhong, M.~Johnson, Y.~Tassa, T.~Erez, and E.~Todorov, ``Value function
  approximation and model predictive control,'' in \emph{{IEEE} Symposium on
  Adaptive Dynamic Programming and Reinforcement Learning {ADPRL}}, 2013, pp.
  100--107.

\bibitem{Mansard18}
N.~Mansard, A.~DelPrete, M.~Geisert, S.~Tonneau, and O.~Stasse, ``Using a
  memory of motion to efficiently warm-start a nonlinear predictive
  controller,'' in \emph{IEEE Int. Conf. on Robotics and Automation {ICRA}},
  2018, pp. 2986--2993.

\bibitem{Ross10}
S.~Ross and D.~Bagnell, ``Efficient reductions for imitation learning,'' in
  \emph{Int. Conf. on Artificial Intelligence and Statistics {AISTATS}}, 2010,
  pp. 661--668.

\bibitem{Ross11}
S.~Ross, G.~J. Gordon, and D.~Bagnell, ``A reduction of imitation learning and
  structured prediction to no-regret online learning,'' in \emph{Int. Conf. on
  Artificial Intelligence and Statistics {AISTATS}}, 2011, pp. 627--635.

\bibitem{Jacobs91}
R.~A. Jacobs, M.~I. Jordan, S.~J. Nowlan, and G.~E. Hinton, ``Adaptive mixtures
  of local experts,'' \emph{Neural Computation}, vol.~3, no.~1, pp. 79--87,
  1991.

\bibitem{Farshidian17}
F.~{Farshidian}, M.~{Neunert}, A.~W. {Winkler}, G.~{Rey}, and J.~{Buchli}, ``An
  efficient optimal planning and control framework for quadrupedal
  locomotion,'' in \emph{IEEE Int. Conf. on Robotics and Automation {ICRA}},
  May 2017, pp. 93--100.

\bibitem{Li04}
W.~Li and E.~Todorov, ``Iterative linear quadratic regulator design for
  nonlinear biological movement systems,'' in \emph{Int. Conf. on Informatics
  in Control, Automation and Robotics {ICINCO}}, 2004, pp. 222--229.

\bibitem{Grandia19}
R.~Grandia, F.~Farshidian, R.~Ranftl, and M.~Hutter, ``Feedback {MPC} for
  torque-controlled legged robots,'' \emph{CoRR}, vol. abs/1905.06144, 2019.

\bibitem{Bertsekas05}
D.~P. Bertsekas, \emph{Dynamic programming and optimal control, 3rd
  Edition}.\hskip 1em plus 0.5em minus 0.4em\relax Athena Scientific, 2005.

\bibitem{Lin92}
L.~J. Lin, ``Self-improving reactive agents based on reinforcement learning,
  planning and teaching,'' \emph{Machine Learning}, vol.~8, pp. 293--321, 1992.

\bibitem{Kingma14}
D.~P. Kingma and J.~Ba, ``Adam: {A} method for stochastic optimization,'' in
  \emph{Int. Conf. on Learning Representations {ICLR}}, 2015.

\bibitem{Reddi18}
S.~J. Reddi, S.~Kale, and S.~Kumar, ``On the convergence of adam and beyond,''
  in \emph{Int. Conf. on Learning Representations ICLR}, 2018.

\end{thebibliography}
\end{document}